\documentclass[11pt]{article} 
\pdfoutput=1
\usepackage[margin=1in]{geometry}
 
\usepackage[dvipsnames]{xcolor}

\usepackage[colorlinks=true, allcolors=magenta]{hyperref}

\usepackage{thmtools}
\usepackage{thm-restate}
\usepackage{amsfonts}
\usepackage{amsmath,amsthm,amssymb}
\usepackage{macros_arxiv}

\usepackage[style=alphabetic,natbib=true,maxbibnames=99,minalphanames=4,maxalphanames=4]{biblatex}
\title{Query Feedback and Teaching for\\Stable and Certifiable Language Generation}
\author{Farnam Mansouri \thanks{f5mansou@uwaterloo.ca. Supported by a NSERC Canada Graduate Scholarship-Doctoral (CGS-D), David Cheriton Scholarship and a Vector Institute Research Grant.}\\University of Waterloo
\and 
Alireza F. Pour\thanks{alireza.fathollahpour@uwaterloo.ca. Supported by a David Cheriton Scholarship and a Vector Institute Research Grant. }\\
 University of Waterloo
 \and
 Grigoris Velegkas\\{Google Research}
 \and
  Sandra Zilles\thanks{Supported through a Canada CIFAR AI Chair at the Alberta Machine Intelligence Institute (Amii), through an NSERC Canada Research Chair, through the New Frontiers in Research Fund under grant no.\ NFRFE-2023-00109 and through the NSERC Discovery Grants program under application no.\ RGPIN-2026-05856.}\\
  University of Regina and Amii}
\date{}

\begin{document}

\maketitle

\begin{abstract}%
Language generation in the limit studies whether a learner observing
an enumeration of an unknown target language can eventually generate
only valid, previously unseen elements. We study stability, introduced
by \citet{kalavasis2025limits}, which requires the generator's infinite
output set to eventually remain unchanged, and introduce the stronger
notion of \emph{certifiability}, which requires the generator to determine from its observations that its current output is valid.
Drawing on teaching sets in learning theory, we introduce a hierarchy of teaching dimensions for generation under three increasingly strict success criteria: correctness, stability, and certifiability. These dimensions measure the minimum number of selected examples whose appearance guarantees the corresponding success condition. 
We also study stability and
certifiability in the membership-query model of
\citet{pmlr-v291-charikar25a}. 
We characterize stable and certifiable generation, both with and
without query feedback, and establish their relationships with uniform
generation and identification in the limit. In particular, we give a
finite-witness characterization of stable generation, analogous to
Angluin's tell-tale characterization of identification, resolving
a gap left open by \citet{kalavasis2025limits}.
We further compare these generation notions with identification in the limit and contrast their teaching requirements on well-studied language families, including regular and erasing pattern languages.
\end{abstract}

\section{Introduction}
Large Language Models (LLMs) exhibit remarkable capabilities across a wide variety of domains.
Yet, despite their empirical success, a rigorous theoretical understanding of how and why these models succeed remains elusive. Because modern LLMs are highly complex systems, developing a mathematical foundation necessitates studying them through clean theoretical abstractions that isolate core phenomena of interest~\citep{malach2023auto,kleinberg2024language,kalai2024calibrated,kalai2025language,joshi2025theory,golowich2025provably}. 

One such abstraction that has recently gained significant traction is the stylized model of \emph{language generation in the limit}, introduced by \citet{kleinberg2024language}. Inspired by the classical inductive inference framework of \citet{gold1967language} and \citet{angluin1980inductive}, this model asks a basic question: when a generative model is trained on a stream of examples from an unknown target language $K$ belonging to a known countable collection of languages $\cL = \{L_1, L_2, \dots\}$ over a domain $\cX$, is there a finite point after which the model consistently produces \emph{novel} and \emph{valid} elements of $K$? Concretely, the problem is formulated as an online game in which an adversary chooses a target language $K \in \cL$ along with an enumeration $x_1, x_2, \dots$ of all elements in $K$. At each round $t$, after observing $S_t = \{x_1, \dots, x_t\}$, the learner outputs a guess---either a single unseen element $y_t \in \cX \setminus S_t$ (\emph{element-based} generation) or an infinite subset $G_t \subseteq \cX$ of candidate outputs (\emph{set-based} generation)---and succeeds in the limit if, after some finite round $t^\star$, all subsequent outputs are valid for $K$.\footnote{The original formulation of \citet{kleinberg2024language} is element-based, but the set-based formulation of subsequent works~\citep{kalavasis2025limits,pmlr-v291-charikar25a,kleinberg2025density,kleinberg2026language} naturally captures the idea of a ``checkpoint'' in training from which onward one can repeatedly sample arbitrarily many novel strings even if training is paused. We adopt the set-based formulation, since is the natural setting for studying \emph{stability}. {Most of} our definitions and results not referring to stability are easily adaptable to the element-based setup.} 
  The main result of \citet{kleinberg2024language}, namely that  
  language generation in the limit is possible for \emph{all} countable collections of languages, stands in stark contrast to the classical negative results for language \emph{identification} in the limit, where the learner must eventually pinpoint the exact index of $K$ in $\cL$ 
  \citep{gold1967language,angluin1980inductive}. This  separation has sparked a vibrant line of work exploring various facets of language generation, including breadth and density~\citep{kalavasis2025limits,pmlr-v291-charikar25a,peale2025representative,kleinberg2025density,kleinberg2026language}, noisy  data~\citep{raman2025generation,bai2026language,mehrotra2025language}, privacy~\citep{mehrotra2026differentially}, and memory constraints~\citep{kleinberg2026language2,flammarion2026space}.

Yet, a crucial practical aspect of language generation has remained largely unexplored: the model of \citet{kleinberg2024language} imposes no requirement on the learner to know \emph{when} it is safe to stop training and start generating. Even though a successful generator is guaranteed to eventually output valid elements after some finite time $t^\star$, this convergence time depends on the unknown target language and enumeration, and the learner itself is never aware of whether $t^\star$ has been reached. In real-world machine learning pipelines, however, practitioners do not train blindly; they continuously track training and validation metrics that inform them about the progress of the run and provide a concrete signal of when training is ``complete'' and the model can be safely deployed. To capture this phenomenon theoretically, we wish to study language generation in the limit by learners that can eventually \emph{certify} that they are generating correctly; we call this \emph{certifiable language generation}. 

Closely intertwined with 
certifiability is a previously studied fundamental property of the learning process, namely \emph{stability} \citep{kalavasis2025limits, kalavasis2025characterizations}. In the original setting of \citet{kleinberg2024language}, the learning algorithm is permitted to change its mind infinitely often, oscillating indefinitely between different candidate generating sets even after all of them are valid subsets of $K$. Stability addresses practical situations in which we would like the training process to eventually converge to a single set to generate from. 
In fact, stability is a necessary condition for certifiability: once a learner can assess that it is safe to stop training, it can immediately freeze its output and stabilize from that point onward. Conversely, stability alone may not suffice to tell us when to stop training, as a stable generator could converge to a valid set at some finite time $t^\star$ without ever knowing that the point of stabilization has been reached. 
\citet{kalavasis2025limits} established a strict separation between stable and non-stable generation even for countable collections, but left open a characterization of when stable generation is possible. Thus, a central goal of our work is to characterize when both \emph{stable} and \emph{certifiable} generation can be achieved.

Moreover, we extend the study of stable and certifiable language generation to two natural settings that were recently introduced in the literature: (i) generation with feedback to active queries by the learner \citep{pmlr-v291-charikar25a,hanneke2026language} and (ii) generation from size-restricted sets of examples provided by the environment \citep{li2026characterizing}, akin to teaching sets \citep{GoldmanK95} in the context of concept learning.

(i) Prior work \citep{pmlr-v291-charikar25a,hanneke2026language} studied settings in which the learner can actively obtain information about the target language $K$ by asking whether strings $x$ of its choice are in $K$. Since stability and certifiability impose string requirements on language generation, we ask the question in which way query feedback can help to fulfill these requirements.

(ii) Both stability and certifiability require that on \emph{every} enumeration of the target language $K$, the learner eventually reaches a point of (certifiable or non-certifiable) stability, though the time and manner in which this point is reached may depend on the specific order in which examples are presented. An even more appealing scenario is when, for every language $K \in \cL$, there exists a finite ``core'' set of examples---which we refer to as a \emph{teaching set}, borrowing terminology from learning theory~\citep{GoldmanK95}---such that as soon as all elements of this set appear in the input stream, the generator is guaranteed to have stabilized (or certified correctness) from that point onward, regardless of the rest of the enumeration. Beyond providing an enumeration-independent guarantee of convergence, the existence of such teaching sets is closely tied to designing data-efficient training curricula. In this work, we seek to understand the exact conditions under which stable generation and certifiable generation can be obtained---with or without teaching sets.

\paragraph{Main Contributions.}
This paper introduces the notion of certifiable language generation as a natural strengthening of stable generation. We show that certifiability is a strictly stronger requirement than stability alone, as part of a set of separation results establishing  a strict hierarchy of the notions of uniform \citep{kleinberg2024language,pmlr-v291-raman25a}, certifiable, stable, and unrestricted language generation, with and without query feedback. Interestingly, under query feedback, stability alone is not a restrictive condition, though certifiability is. Notably, some of our separations are witnessed by the well-studied class of pattern languages \citep{angluin80patterns,shinohara82}.

To gain a deeper understanding of the additional difficulty that certifiability imposes over stability, we give characterizations of generation under these two notions, thereby answering an open question of \citet{kalavasis2025limits,kalavasis2025characterizations} asking for a characterization of stable generation. We complement these results by analogous characterizations in the query-augmented learning setting.

Our characterizations in the query-free setting rely on the structure of finite witness sets, akin to tell-tales used in characterizations in the identification setting~\citep{angluin1980inductive}. Such finite witness sets are closely related to teaching sets studied in learning theory \citep{GoldmanK95}.
While teaching sets for language generation were implicitly introduced by \cite{li2026characterizing}, we initiate their study in the settings of stable and certifiable generation, again establishing a hierarchy across various notions of generation.  In contrast to Li et al.'s (\citeyear{li2026characterizing}) result, stating that generation for countable language collections can always be achieved with teaching sets of size 1, we use countable classes to exhibit arbitrarily large separations between teaching set sizes for certifiable generation, stable generation, and generation in the limit. This shows that, in general, larger teaching sets are needed to satisfy stronger requirements on the learning process. 

Incidentally, we exhibit a relationship between the size of stable generation teaching sets and the size of finite sets with the so-called \emph{non-clashing}\/ property, which is known in the context of machine teaching \citep{FallatKSSZ23}, identification in the limit \citep{pos2024mousawi}, and sample compression \citep{KuzminW07}. This relationship mirrors an analogous relationship between complexity parameters for identification in the limit \citep{pos2024mousawi}.


\section{Preliminaries}

Let $\Nat$ be the set of natural numbers and $\Nat_0:= \Nat \cup \{0\}$. Fix any countable domain $\cX$. For any sequence $\sigma\in \cX^*$ we denote by $\supp(\sigma) \subseteq \cX$ the set of elements in $\sigma$. A language $L$ is consistent with 
$\sigma$ if $\supp(\sigma)\subseteq L$. For any $\sigma\in\cX^*$ and finite set $T$, we say $\sigma$ \emph{contains} $T$ if $T\subseteq\supp(\sigma)$. For any infinite sequence $\bar{\sigma} = (x_1,x_2,\ldots)$ and natural number $n\in \bN$ we denote by $\bar{\sigma}_n = (x_1,\ldots,x_n)$ the prefix of length $n$ of $\bar{\sigma}$. The concatenation of finite sequences $\sigma$ and $\sigma'$ is denoted by $\sigma\cdot\sigma'$. 




Let $\cL$ be a collection of infinite languages over $\cX$. Denote by $\cB_{\cX}$ the collection of all infinite subsets of $\cX$. A \emph{generator}\/ is a sequence $\cG=(\cG_n)_{n\in\Nat_0}$, where, for every $n\in\Nat_0$, $\cG_n:\cX^n\rightarrow\cB_{\cX}$.
For simplicity, we will write $\cG(\sigma)$ for $\cG_n(\sigma)$, when $\sigma$ is a finite sequence of length $n$.




A \emph{repetition-free enumeration}\/ of an infinite language $L$ is an (infinite) sequence of elements in $L$ in which each element of $L$ appears exactly once. Every prefix of a repetition-free enumeration is a repetition-free finite sequence. In generation in the limit, we assume that, at each time step, a previously unseen training datum is presented (alternatively, one may allow for repetitions, which can be ignored by the learner). Therefore, for simplicity, when we use the term \emph{enumeration}, we implicitly refer to a repetition-free enumeration. 

\begin{definition}[\cite{kleinberg2024language}]\label{def:generation}
A collection $\cL$ is \emph{generatable in the limit} if there is a generator
$\cG$ such that, for every $K\in\cL$ and every enumeration $\bar{\sigma}$ of $K$, there
exists $n^*\in\Nat_0$ satisfying $\cG(\bar{\sigma}_n)\subseteq K$ for every $n\geq n^*.$
\end{definition}

\begin{definition}[\cite{kleinberg2024language, pmlr-v291-raman25a}]\label{def:uniform}
For $i\in \Nat_0$, a collection $\cL$ is \emph{$i$-uniformly generatable} if
there is a generator $\cG$ such that 
for every $K \in \cL$, every enumeration $\bar{\sigma}$ of $K$, and every $n \geq i$ we have
$\cG(\bar{\sigma}_n)\subseteq K$. 
It is \emph{uniformly generatable} if it is $i$-uniformly generatable for some
$i\in\Nat_0$.
\end{definition}



We contrast the notion of generation with the notion of identification, which requires eventual recovery of the whole target, rather than of an infinite subset.

\begin{definition}[\cite{gold1967language}]\label{def:identification}
A collection $\cL$ is \emph{identifiable in the limit} if there exists a
mapping $\mathcal A:\cX^*\to\cL$ \footnote{\citet{gold1967language} required $\mathcal A$ to be a computable function, but this paper ignores computability aspects. Whenever we use the term ``identifiable in the limit'', we refer to the version without computability requirements.}
such that, for every $K\in\cL$ and every
enumeration $\bar{\sigma}$ of $K$, there is $n^*\in\Nat_0$ satisfying $\mathcal A(\bar{\sigma}_n)=K$ for every $n\geq n^*$. 

\end{definition}

Language collections that are identifiable in the limit can be characterized by assignments of finite sets with a certain structural property, to be defined next.


\begin{definition}
[\cite{angluin1980inductive,FallatKSSZ23}]
\label{def:telltale}
Let $\cL$ be a language collection, and let $(T_L)_{L\in\cL}$ be a family of finite sets such that $T_L\subseteq L$ for every $L\in\cL$. 

\begin{enumerate}[label=(\roman*)]
    \item 
    $(T_L)_{L\in\cL}$ is called a \emph{tell-tale assignment} if $T_L\subseteq L'\subseteq L$
        implies $L'=L$, whenever 
    $L,L'\in\cL$.

    \item 
    $(T_L)_{L\in\cL}$ is called \emph{non-clashing} if $T_L\not\subseteq L'$ or $T_{L'}\not\subseteq L$ for any
    $L,L'\in\cL$ with $L\ne L'$.
\end{enumerate}
\end{definition}

\begin{restatable}[
{\cite{angluin1980inductive,pos2024mousawi}}]{theorem}{restateIdentificationTelltales}
\label{thm:identification-telltales}
For every countable language collection $\cL$, the following statements
are equivalent:
\begin{enumerate}[label=(\roman*)]
    \item $\cL$ is identifiable in the limit.
    \item $\cL$ admits a tell-tale assignment of finite sets.
    \item $\cL$ admits a non-clashing assignment of finite sets.
\end{enumerate}
\end{restatable}

These characterizations also connect identification in the limit to
machine teaching. In the machine teaching literature, the positive non-clash teaching
dimension ($\NCTD^+$) measures the smallest uniform bound on the sizes of sets in a non-clashing assignment. 



\begin{definition}
[\cite{FallatKSSZ23}]
\label{def:nctd}
The \emph{positive non-clash teaching dimension} of a language
collection $\cL$ is defined as
$\NCTD^+(\cL)
    =
    \inf\nolimits_{(T_L)_{L\in\cL}\ \mathrm{non\text{-}clashing}}
    \sup\nolimits_{L\in\cL}|T_L|$.
\end{definition}

We will later adapt the notion of tell-tale assignment to stable language generation and then show that it is useful to likewise measure the smallest uniform bound on the sizes of such tell-tale sets.

\section{Stable and Certifiable Language Generation}\label{sec:certifiability}

One of our central goals is to model and study language generation with a learner that can eventually \emph{certify}\/ that it is sufficiently well trained to guarantee correct generation from the target language. Clearly, once such a certificate can be issued, the generator may choose to \emph{stabilize}\/ on its current output, i.e., no further training is necessary.  
Stability, which requires the generator to stop changing its output after a finite time, was recently introduced by \citet{kalavasis2025limits}. We will show below that this condition is strictly weaker than our proposed condition of certifiability. 


\begin{definition}[\cite{kalavasis2025limits}]\label{def:stable}
A collection $\cL$ is \emph{stably generatable} if there exists a generator
$\cG$ such that, for every $K\in\cL$ and every enumeration $\bar{\sigma}$ of $K$, there
is $n^*\in\Nat_0$ such that for every $n\geq n^*$, we have $\cG(\bar{\sigma}_n)=\cG(\bar{\sigma}_{n^*})\subseteq K$.
\end{definition}

\begin{definition}
\label{def:certifiable}
A collection $\cL$ is \emph{certifiably generatable} if there exist a
generator $\cG$ and a certificate $c:\cX^*\to\{0,1\}$ satisfying:
\begin{enumerate} [label=(\roman*)]
\item \emph{Soundness.} For every $\sigma \in \cX^*$, if $c(\sigma)=1$, then
$\cG(\sigma)\subseteq L$ for every $L\in\cL$ consistent with $\sigma$.
\item \emph{Eventual certification.} For every $K\in\cL$ and every
enumeration $\bar{\sigma}$ of $K$, there is $n^*\in\Nat_0$ such that $c(\bar{\sigma}_{n^*})=1$.
\end{enumerate}
\end{definition}

\begin{remark}
    After certification, the generator may retain its certified output on every extension of the prefix $\bar{\sigma}_{n^*}$ on which the certificate was obtained. This preserves soundness, because every language consistent with the extension was also consistent at the time of certification. Thus a certifiable generator can be made stable. Theorem~\ref{thm:generation-hierarchy} will show that the converse is not true in general.
\end{remark}

\subsection{Finite-Witness Characterizations} \label{sec:stab-cert}

 Pinpointing the structural properties of language classes that are generatable in a given model provides insights into the nature of standardized underlying learning processes. To this end, we derive finite-witness characterizations 
 for stable and certifiable generation, similar to the analogous characterizations for identification, see  Theorem~\ref{thm:identification-telltales}. Our finite witnesses for stable generation satisfy a weakened non-clashing condition, allowing distinct languages to receive the same witness.
 

\begin{restatable}{theorem}{restateStableChar}\label{thm:stable-char}
A collection $\cL$ is stably generatable if and only if there is an assignment
of finite sets $T_L\subseteq L$, $L\in\cL$, satisfying:
\begin{equation}\label{eq:stable-fiber}
\mbox{(i) For every }L\in\cL,\qquad\left|\bigcap\nolimits_{L'\in\cL\ :\ T_{L'}=T_L}L'\right|=\infty.
\end{equation}
\begin{equation}\label{eq:stable-noclash}
\mbox{(ii) For every }L,L'\in\cL,\qquad T_L\subseteq L',\quad T_{L'}\subseteq L
\quad\Longrightarrow\quad T_L=T_{L'}.
\end{equation}
\end{restatable}
\begin{proof}
\emph{If direction.}
Define the generator $\cG$ as follows. For every finite sequence $\sigma$, choose an arbitrary language $L\in\cL$ with \( T_L\subseteq \supp(\sigma)\subseteq L, \) and define 
\[ \cG(\sigma) := \bigcap\nolimits_{L'\in\cL:T_{L'}=T_L}L'. \] 
If no language $L\in\cL$ with \( T_L\subseteq \supp(\sigma)\subseteq L \) exists, define $\cG(\sigma)$ arbitrarily.  Fix a target language $K\in\cL$. Since $T_K$ is finite, every enumeration of $K$ has a prefix containing $T_K$. Consider any extension $\sigma$ of such prefix, satisfying \( T_K\subseteq \supp(\sigma) \subseteq K. \) Suppose that some $L\in\cL$ also satisfies $T_L\subseteq \supp(\sigma) \subseteq L$. Then 
$T_L\cup T_K\subseteq \supp(\sigma) \subseteq L\cap K$.
By condition (ii), this implies $T_L=T_K$.
Therefore, every language that can be selected by the generator at this point has the same associated set as $K$. Consequently, 
$\cG(\sigma) =\bigcap\nolimits_{L\in\cL:T_{L}=T_K}L$,
which is infinite by condition (i), is contained in $K$, and remains unchanged on every subsequent prefix. Hence, $\cL$ is stably generatable.

\emph{Only-if direction.}
Consider a generator $\cG$ that stably generates $\cL$. 
Consider a fixed ordering $\prec$ of $\cX$.
For every $L\in\cL$, we construct a finite sequence $\sigma_L \in L^*$  in phases and then take $T_L=\supp(\sigma_L)$.

%

Initialize $\tau_{L,0}=\emptyseq$. In phase $n\geq1$, let $z_{L, n}$ be the first element in $L\setminus\supp(\tau_{L,n-1})$ according to the ordering $\prec$, and set $\eta_{L,n}=\tau_{L,n-1}\cdot
z_{L,n}.$
Consider all finite repetition-free sequences
$\rho$ over $L\setminus\supp(\eta_{L,n})$. If there is such a sequence with
$\cG(\eta_{L,n}\cdot\rho)\ne\cG(\eta_{L,n}),$
let $\rho_{L,n}$ be the first such sequence in
the shortlex order on finite sequences induced by the
ordering $\prec$ of $\cX$, and set
$\tau_{L,n}=\eta_{L,n}\cdot\rho_{L,n}$. Otherwise, stop and define $\sigma_L:=\eta_{L,n}$ and $ T_L:=\supp(\sigma_L)$.

Because each phase includes the least unseen element in the fixed order,
an infinite run would enumerate $L$ with $\cG$ changing its output infinitely many times,
contradicting stable generation. Therefore the process terminates after a finite number of phases, and $T_L$ is finite.

By the stopping condition, no continuation over
$L\setminus T_L$ changes the output. Complete $\sigma_L$ to an
enumeration of $L$. The generator outputs $\cG(\sigma_L)$ on every
subsequent prefix, and correctness on this enumeration implies that this
set is contained in $L$. Thus
\begin{equation}\label{eq:stable-locking}
\cG(\sigma_L\cdot\rho)=\cG(\sigma_L)\subseteq L
\qquad\text{for every continuation }\rho\text{ over }L\setminus\supp(\sigma_L).
\end{equation}

We next compare the constructions for two languages $L,L'\in\cL$ such that $T_L\cup T_{L'}\subseteq L\cap L'.$
For the sake of contradiction, suppose $\sigma_{L} \neq \sigma_{L'}$.
Let $n\in\Nat$ be the largest index such that \( \tau_{L,n - 1}=\tau_{L',n - 1}. \) There are two cases. First observe that since $z_{L, n }, z_{L', n} \in L$, we have $z_{L, n} \fmreplace{\prec}{\preceq} z_{L', n}$. Similarly $z_{L', n} \fmreplace{\prec}{\preceq} z_{L, n}$. Thus, $z_{L, n} = z_{L', n}$. This indicates that $\eta:= \eta_{L, n} = \eta_{L', n}$. 

\textbf{Case 1.} Construction of $L$ or $L'$ stops at phase $n$. Without loss of
generality, suppose that the construction for $L$ stops and note that the construction of $L'$ cannot stop as otherwise $\sigma_L=\sigma_{L'}$. The extension $\rho_{L', n}$ satisfies $\cG(\eta \cdot \rho_{L', n}) \neq \cG(\eta)$. It is also consistent with $L$ and is a sequence of $L' \setminus \supp(\eta)$. Thus, it is also a sequence of $L \setminus \supp(\eta)$. This contradicts $L$ stopping.


\textbf{Case 2.} Both constructions continue, but $\rho:= \rho_{L, n}$ is not equal to $\rho' := \rho_{L', n}$. Without loss of generality, suppose that
$\rho$ comes first in the shortlex order. Observe that $\supp(\rho) \subseteq L \setminus \supp(\eta)$ and that
$\cG(\eta\cdot\rho)\ne\cG(\eta).$
Thus $\rho$ was an available continuation of $\eta$ for $L'$ that would change the output of $\cG$ as
well. Hence, it should have been chosen before $\rho'$ contradicting the rule
of selecting the first such continuation in shortlex order.

Since neither case is possible, the two constructions agree in
every phase and thus $\sigma_L=\sigma_{L'}$. Therefore,
$T_L=T_{L'}$. This proves condition (ii).

Finally, fix $L\in\cL$. Whenever $T_{L'}=T_L$, these sets are mutually
consistent with $L$ and $L'$, so the preceding argument gives
$\sigma_{L'}=\sigma_L$. Applying \eqref{eq:stable-locking} to each such
$L'$ yields
$
\cG(\sigma_L)\subseteq
\bigcap_{L'\in\cL:T_{L'}=T_L}L'$.
Since $\cG(\sigma_L)$ is infinite, this proves
condition (i) and completes the proof.
\end{proof}

The following characterization of certifiable generatability is not difficult to prove (see Appendix~\ref{app:stab-cert}); the standardized generator for the ``if'' direction always simply outputs the intersection of all languages in $\cL$ that are consistent with the observed data whenever the intersection is infinite.

\begin{restatable}{theorem}{restateCertChar}\label{thm:cert-char}
A collection $\cL$ is certifiably generatable if and only if, for every
$K\in\cL$, there is a finite set $T_K\subseteq K$ such that
\begin{equation}\label{eq:certificate}
\left|\bigcap\nolimits_{L\in\cL:T_K\subseteq L}L\right|=\infty.
\end{equation}
\end{restatable}



These characterizations now help to establish the following hierarchy of requirements on language generation; see also Figure~\ref{fig:complete-gen-hierarchy} for a complete hierarchy including notions introduced in Section ~\ref{sec:queries}. 

\begin{restatable}{theorem}{restateGenerationHierarchy}\label{thm:generation-hierarchy}
For every collection $\cL$ of infinite languages, the following implications hold.
\[
\begin{aligned}
\cL\text{ is uniformly generatable}
&\quad\Longrightarrow\quad\cL\text{ is certifiably generatable.}\\
\cL\text{ is certifiably generatable}&\quad\Longrightarrow\quad\cL\text{ is stably generatable.}\\
\cL\text{ is stably generatable}&\quad\Longrightarrow\quad\cL\text{ is generatable in the limit}.
\end{aligned}
\]
Every implication is strict.
\end{restatable}

\begin{proof} \emph{(Sketch.)} Implications $A\Rightarrow B$ are shown by standardizing generators witnessing $A$ so as to satisfy the criteria for $B$. Strictness is shown using the class of all languages $\{2,4,\ldots,2i\}\cup\{2^i(2k-1)-1:k\in\Nat\}$ for the first implication; the class of all languages $\mathbb{N}\setminus\{i\}$ for the second implication; and the class of all languages $\mathbb{N}\setminus\{i\}$ plus the language $\mathbb{N}$ (which was also used by \cite{kalavasis2025characterizations}) for the third implication. See Appendix~\ref{app:stab-cert} for a full proof.
\end{proof}

\subsection{Separations Witnessed by Families of Pattern Languages}

It is not difficult to construct language families that are stably or certifiably generatable, but are not identifiable in the limit. For instance, any language collection with a common infinite intersection is \(1\)-uniformly generatable. Notably though, we will show that similar separations also arise for natural structured language families, specifically, for classes of (erasing) pattern languages, which we will introduce next, following the original definitions by \citet{angluin80patterns} and \citet{shinohara82}.

Let $\Sigma$ be a finite alphabet with $|\Sigma|\ge 2$ and let $\mathcal Z$ be a countable
set of variables disjoint from $\Sigma$. A \emph{pattern} is a finite word
$\alpha\in(\Sigma\cup\mathcal Z)^+$ containing at least one variable.\footnote{Angluin's original definition allows a pattern not to have any variable, in which case its language is a singleton. Since the formal framework of language generation deals only with infinite languages, we exclude such patterns.}
An \emph{erasing substitution} fixes the symbols of $\Sigma$ and maps each
variable to an arbitrary word in $\Sigma^*$, possibly the empty word
$\varepsilon$. It extends to a homomorphism on patterns. The corresponding
\emph{erasing pattern language} is
\[
L_E(\alpha)=\{h(\alpha):h\text{ is an erasing substitution}\}.
\]
We write $\eFEP_\Sigma$ for the collection of these languages, and
$\eTFEP_\Sigma$ for the subcollection defined by nonempty terminal-free
patterns $\alpha\in\mathcal Z^+$. All these languages are infinite.

While uniform generation is a very strong requirement, for binary alphabets, the class $\eTFEP_\Sigma$ is a ``natural'' language family that is 2-uniformly generatable, but not identifiable in the limit. 2-uniform generatability is shown by constructing an infinite subset of the target language $L$ directly from the first non-empty word observed in the enumeration of $L$; see Appendix~\ref{app:stab-cert} for details.

\begin{restatable}{theorem}{restateTfPatterns}\label{thm:tf-patterns}
The collection $\eTFEP_\Sigma$ is certifiably generatable
and $2$-uniformly generatable, but is not $1$-uniformly generatable, and, when $|\Sigma|=2$, not identifiable in the limit.
\end{restatable}

Allowing terminal symbols results in a class that is no longer uniformly, but still stably generatable. The proof details are more involved than for Theorem~\ref{thm:tf-patterns} and are also provided in Appendix~\ref{app:stab-cert}.

\begin{restatable}{theorem}{restatePatterns}\label{thm:patterns}
The collection $\eFEP_\Sigma$ is stably
generatable, but is not uniformly generatable, and, when  $|\Sigma|\in\{2,3,4\}$, not identifiable in the limit.
\end{restatable}

\section{Stable and Certifiable Language Generation with Query Feedback}\label{sec:queries}

Since stable and certifiable generation is difficult to achieve even for some structurally simple language classes, it is natural to ask whether feedback can help the learner. We therefore define and analyze stability and certifiability in the membership-query model of \citet{pmlr-v291-charikar25a}, which was later studied by \citet{hanneke2026language}. In this model, at round $n$, the generator observes
the initial segment $\bar{\sigma}_n=(x_1,\ldots,x_n)$ of an enumeration $\bar{\sigma}$ of the target language $K$, chooses a query
$y_n\in\cX$ based on $\bar{\sigma}_n$ and the previous query answers, and receives $a_n=\mathbf 1\{y_n\in K\}$.
The generator's output may depend on the full interaction history
$(\bar{\sigma}_n,\bar{\alpha}_n)$, where $\fmreplace{\bar{\alpha}_n}{\bar{\alpha}}=(a_1,a_2,\ldots)$. The query strategy is
fixed as part of the generating procedure, and hence the queried points can be
reconstructed from the interaction history. 
We first recall the set-output version of generation with query feedback
introduced by \citet{hanneke2026language}.

\begin{definition}[\cite{pmlr-v291-charikar25a}]\label{def:query-generation}
A query-feedback generator consists of a query strategy and a sequence of maps $\cG = (\cG_n)_{n \in \Nat_0}$ where for every $n \in \Nat_0, \ \cG_n:\cX^n\times\{0,1\}^n\to\cB_{\cX}.$ It \emph{generates $\cL$ in the limit with query feedback} if, for every
$K\in\cL$ and every enumeration $\bar{\sigma}$ of $K$, there exists $n^* \in \Nat_0$
satisfying $\cG(\bar{\sigma}_n,\bar{\alpha}_n)\subseteq K$ for all $n \geq n^*$.
\end{definition}

We generalize this model to stable and certifiable generation with query feedback.

\begin{definition}\label{def:stable-query}
A collection $\cL$ is \emph{stably generatable with query feedback} if there
is a query-feedback generator $\cG$ such that, for every $K\in\cL$ and every
enumeration $\bar{\sigma}$ of $K$, there exists $n^*\in\Nat_0$ satisfying
$\cG(\bar{\sigma}_n,\bar{\alpha}_{n})=\cG(\bar{\sigma}_{n^*},\bar{\alpha}_{n^*})\subseteq K$ for every $n\geq n^*$.
\end{definition}

The notion of certifiability is defined analogously to the no-feedback setting,
except that the certificate may now depend on both the observed positive
examples and the answers to the membership queries. Thus, when certification
occurs, the current output must be valid for every language consistent with the
entire interaction history, including the answers to the queries.
To formalize this, for finite disjoint sets $P,N\subseteq\cX$, define the
\emph{signed version space}
\begin{equation*}
\cV_{\cL}(P,N)
=
\{L\in\cL:P\subseteq L,\ L\cap N=\varnothing\}.
\end{equation*}
At round $n$, let $P_n=\supp(\bar{\sigma}_n)\cup\{y_t:t\leq n,\ a_t=1\}$ and $N_n=\{y_t:t\leq n,\ a_t=0\}$.
Then $\cV_{\cL}(P_n,N_n)$ is exactly the collection of languages in $\cL$
that are consistent with all observed examples and all membership-query
answers up to round $n$.

\begin{definition}\label{def:cert-query}
A collection $\cL$ is \emph{certifiably generatable with query feedback} if
there exist a query-feedback generator $\cG$ and certificate
$c: \bigcup_{n\in\Nat_0} (X^n \times \{0, 1\}^n) \to\{0,1\}$ satisfying:
\begin{enumerate}[label=(\roman*)]
\item On every realizable finite interaction, if
$c(\bar{\sigma}_n,\bar{\alpha}_n)=1$, then
$\cG(\bar{\sigma}_n,\bar{\alpha}_n)\subseteq
\bigcap_{L\in\cV_{\cL}(P_n,N_n)}L$.
\item For every $K\in\cL$ and every
enumeration $\bar{\sigma}$ of $K$, the induced interaction satisfies
$c(\bar{\sigma}_{n^*},\bar{\alpha}_{n^*})=1$ for some $n^*\in\Nat_0$.
\end{enumerate}
\end{definition}

Query-feedback generation was characterized by \citet{hanneke2026language} in terms of the existence of a countable inner-cover, defined below. Interestingly, stable query-feedback generation admits the same characterization; the reader is referred to Appendix~\ref{app:query} for a proof of Theorem~\ref{thm:stable-query}.

\begin{definition}\label{def:inner-cover}
A \emph{countable inner-cover} of $\cL$ is a countable family
$\cC=\{C_1,C_2,\ldots\}\subseteq\cB_{\cX}$ such that every $K\in\cL$
contains some $C_i$. The sets $C_i$ need not belong to $\cL$.
\end{definition}

\begin{restatable}{theorem}{restateStableQuery}\label{thm:stable-query}
For a collection $\cL$, the following are equivalent:
\begin{enumerate}[label=(\roman*)]
\item $\cL$ is stably generatable with query feedback.
\item $\cL$ is generatable in the limit with query-feedback  generator (with set-valued output).
\item $\cL$ has a countable inner-cover.
\end{enumerate}
\end{restatable}

We next characterize certifiable generation with query feedback in terms of finite positive and negative witnesses; the proof is again provided in Appendix~\ref{app:query}.

\begin{restatable}{theorem}{restateCertifiableQueryFeedback}\label{thm:certifiable_query_feedback}
A collection $\mathcal L$ is certifiably generatable with query feedback if and only if, for every $K\in\mathcal L$, there
exist finite sets $P_K\subseteq K$ and $N_K\subseteq\mathcal X\setminus K$
such that
$\left|
\bigcap\nolimits_{L\in\mathcal V_{\mathcal L}(P_K,N_K)}L
\right|=\infty$.
\end{restatable}

Figure~\ref{fig:complete-gen-hierarchy} summarizes the relationships
between the notions studied in this paper. 
In particular, it shows that certifiable generation with query feedback is fundamentally different from both stable and ordinary query-feedback generation. We show that there exists a collection that is stably generatable in the limit even without feedback (and identifiable in the limit), but not certifiably generatable with query feedback. Thus, although membership queries allow a learner to observe negative examples, they do not suffice to certify every stably generatable collection. See Appendix~\ref{app:query} for proof details.

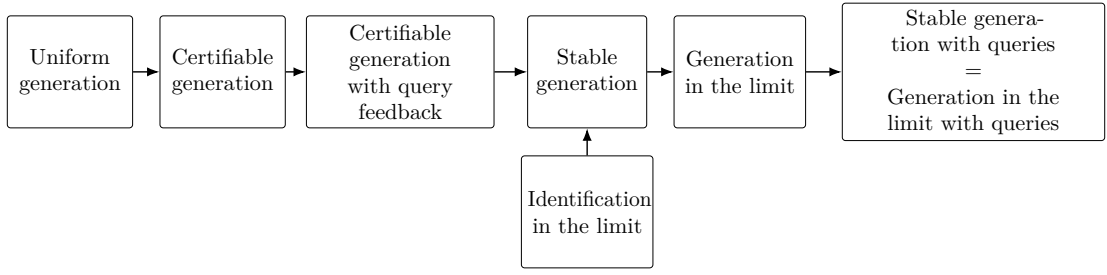
\begin{figure}[t]
\centering
\resizebox{0.88\linewidth}{!}{%
\begin{tikzpicture}[
  every node/.style={font=\small,align=center},
  notion/.style={
    draw,rounded corners=1.5pt,
    minimum height=18mm,inner sep=3pt
  },
  arr/.style={-{Latex[length=2mm]},line width=0.65pt}
]
\node[notion,text width=18mm] (U) at (0,0)
  {Uniform\\generation};
\node[notion,text width=18mm] (C) at (2.45,0)
  {Certifiable\\generation};
\node[notion,text width=28mm] (CQ) at (5.3,0)
  {Certifiable generation\\with query feedback};
\node[notion,text width=17mm] (S) at (8.3,0)
  {Stable\\generation};
\node[notion,text width=19mm] (G) at (10.75,0)
  {Generation\\in the limit};
\node[notion,text width=40mm] (SQ) at (14.5,0)
  {Stable generation with queries\\$=$\\
   Generation in the limit with queries};

\node[notion,text width=19mm] (I) at (8.3,-2.25)
  {Identification\\in the limit};

\draw[arr] (U.east) -- (C.west);
\draw[arr] (C.east) -- (CQ.west);
\draw[arr] (CQ.east) -- (S.west);
\draw[arr] (S.east) -- (G.west);
\draw[arr] (G.east) -- (SQ.west);
\draw[arr] (I.north) -- (S.south);
\end{tikzpicture}%
}
\caption{Hierarchy of language generation. A path from $A$ to $B$ exists \emph{if and only if}\/ every language family having a generator/learner of type $A$ also has a generator/learner of type~$B$.}
\label{fig:complete-gen-hierarchy}
\end{figure}

\begin{restatable}{theorem}{restateCompleteGenHierarchy}\label{thm:complete-gen-hierarchy}
In Figure~\ref{fig:complete-gen-hierarchy}, there is a directed path from $A$ to $B$ if and only if every language family having a generator/learner of type $A$ also has a generator/learner of type $B$. 
\end{restatable}

\section{Stable and Certifiable Language Generation with a Teacher} \label{sec:teaching}

Stability and certifiability conditions require a generator to eventually stabilize, regardless of the input enumeration chosen for the target language. It is natural to ask whether, for any language $L$ in a stably or certifiably generatable family $\cL$, there is a finite set $T\subseteq L$ such that, as soon as $T$ appears in the input enumeration, a generator can stabilize. We call such $T$ a \emph{teaching set}\/ for $L$ (with respect to $\cL$). One can then measure the complexity of training a generator for $\cL$ by the supremal size of a smallest teaching set for $L$, where the supremum---called \emph{teaching dimension}\/---is taken over all languages $L\in\cL$. This approach parallels machine teaching for language identification \citep{GoldmanK95}, except that the objective here is successful generation rather than exact identification of the target language. 

We begin by defining notions of teaching set and teaching dimension for generation in the limit, stable generation in the limit, and finally certifiable generation in the limit.
\cite{li2026characterizing} already established that finite teaching sets always exist for generation in the limit, without stability and certifiability conditions. They introduced a parameter called \emph{positive separation width}, which coincides with our notion of \emph{generation teaching dimension (GTD)}, when the latter is finite.

\begin{definition}\label{def:gtd}
For a generator $\cG$ and $K\in\cL$, a finite set $T\subseteq K$ is a
\emph{generation teaching set} for $K$ if $\cG(\sigma)\subseteq K$ for every $\sigma \in K^*$ containing $T$.
Let $\mathcal T_{\cG}(K)$ be the collection of all such sets. Define 
\[
\GTD(\cL,\cG)=\sup\nolimits_{K\in\cL}\inf\nolimits_{T\in\mathcal T_{\cG}(K)}|T|
\qquad\text{and}\qquad
\GTD(\cL)=\inf\nolimits_{\cG}\GTD(\cL,\cG).
\]
Here we assume that the infimum taken over the empty set is $\infty$.
\end{definition}

Note that we require the generator to generate successfully on every input that contains the teaching set, rather than only on the teaching set itself. This is analogous in
spirit to the collusion-freeness requirements used in machine teaching which prevent the teacher and learner from using coding tricks \citep{GoldmanM96,FallatKSSZ23}. Without this requirement, a teacher could, for example, encode a representation of the target language in the teaching set, which the learner (or generator) would be able to decode. Successful learning from any enumeration containing the teaching set prevents such coding tricks. 

\cite{li2026characterizing} showed that every collection that is generatable in limit admits a generator with finite generation teaching sets; for countable collections, the size of these sets can be bounded by 1.

\begin{restatable}[\cite{li2026characterizing}]{theorem}{restateGtd}\label{thm:gtd}
    A language collection $\cL$ is generatable in the limit if and only if there is a generator $\tilde \cG$ such that every $K \in \cL$ has a finite generation teaching set with respect to $\tilde \cG$. Moreover, for every countable collection $\cL$, we have $\GTD(\cL)\leq1$.
\end{restatable}

We next adapt teaching to stable generation. In addition
to guaranteeing correctness after the teaching set has appeared, we require the
generator's output to remain unchanged under every subsequent target-consistent
continuation, yielding the notion of Stable Generation Teaching Dimension.

\begin{definition}\label{def:sgtd}
For a generator $\cG$ and $K\in\cL$, a finite set $T\subseteq K$ is a
\emph{stable generation teaching set} if, for every $\sigma\in K^*$
containing $T$ and every $\sigma'\in(K\setminus\supp(\sigma))^*$, we have $\cG(\sigma)=\cG(\sigma\cdot \sigma')\subseteq K.$ Let $\mathcal T^{\mathrm{st}}_{\cG}(K)$ be the collection of these sets.
Define
\[
\SGTD(\cL,\cG)=\sup\nolimits_{K\in\cL}
  \inf\nolimits_{T\in\mathcal T^{\mathrm{st}}_{\cG}(K)}|T|
\qquad \text{and} \qquad
\SGTD(\cL)=\inf\nolimits_{\cG}\SGTD(\cL,\cG).
\]
Here we assume that the infimum taken over the empty set is $\infty$.
\end{definition}

Finally, we define a notion of Certifiable Generation Teaching Dimension.
Here, a teaching set must provide enough information to certify an infinite set
of outputs that is valid for every language consistent with the teaching set.

\begin{definition}\label{def:cgtd}
For $K\in\cL$, a finite set $T\subseteq K$ is a \emph{certifiable generation
teaching set} if
$\left|\bigcap_{L\in\cL:T\subseteq L}L\right|=\infty$. Define
\[
\CGTD(\cL)=\sup\nolimits_{K\in\cL}\inf\left\{|T|:
T\subseteq K,\quad
\left|\bigcap\nolimits_{L\in\cL:T\subseteq L}L\right|=\infty\right\}.
\]
\end{definition}


\subsection{Relationships and Basic Properties}\label{sec:relationships}

It is easy to see that $\GTD(\cL)\leq\SGTD(\cL)\leq\CGTD(\cL)$ for every language collection $\cL$. This subsection addresses the question by how much these parameters can differ. Since teaching sets for each of the studied notions of generation bear some conceptual similarity with the non-clashing finite witness sets from our characterization of stable generatability (Theorem~\ref{thm:stable-char}), we will additionally compare our three teaching dimension parameters to the size of the finite witnesses therein:


\begin{definition}\label{def:ASGTD}
Let $\mathfrak A(\cL)$ be the set of assignments
$\mathbf T=(T_L)_{L\in\cL}$ of finite sets $T_L\subseteq L$ satisfying
\eqref{eq:stable-fiber} and \eqref{eq:stable-noclash}. Define the Stable Non-Clashing Positive Witness Size by
\[
\STT(\cL):=\inf\nolimits_{\mathbf T\in\mathfrak A(\cL)}
                 \sup\nolimits_{L\in\cL}|T_L|.
\]
\end{definition}


In the identification setting,  \cite{pos2024mousawi} noted that non-clashing assignments are always tell-tale assignments, and thus that the smallest uniform bound on the tell-tale size never exceeds $\NCTD^+$. Interestingly, in a setting without computability requirements, tell-tales can be construed as teaching sets for identification in the same sense as stable generation teaching sets are teaching sets for stable generation. (If the learner has seen the tell-tale for a consistent language $L$, it is safe to output $L$: If $L$ is incorrect, the target cannot be a proper subset of $L$, as it contains the tell-tale for $L$. Hence, some string outside $L$ will be enumerated at some point, at which the learner will be allowed to change its mind again.) This motivates the following definition of \emph{Teaching Dimension for Identification}.
\begin{definition}\label{def:ITD}
$\ITD(\cL):=\inf\nolimits_{\mathbf T}\sup\nolimits_{L\in\cL}|T_L|$, where $\mathbf{T}$ ranges over all tell-tale assignments for $\cL$. 
\end{definition}
Thus, according to \cite{pos2024mousawi},  $\ITD$ never exceeds $\NCTD^+$ (but it may be smaller). Theorem~\ref{thm:dimension-hierarchy} establishes an analogous
relationship between $\SGTD$ and $\STT$.




\begin{restatable}{theorem}{restateDimensionHierarchy}\label{thm:dimension-hierarchy}
$\GTD(\cL)\leq\SGTD(\cL)\leq\STT(\cL)\leq\CGTD(\cL)$ for every collection $\cL$.
Moreover, there are countable collections $\cL_1,\cL_2,\cL_3$ such that $\GTD(\cL_1)\leq 1$ and $\SGTD(\cL_1)=\infty$, $\SGTD(\cL_2) \leq 1$ and $\STT(\cL_2)=\infty$, as well as $\STT(\cL_3)\leq 1$ and $\CGTD(\cL_3)=\infty$.
\end{restatable}

\begin{proof} \emph{(Partial sketch.)} Proving the chain of inequalities is fairly straightforward. 
The first separation is obtained with $\cL_1=\{\Nat\}\cup\{\Nat\setminus\{i\}:i\in\Nat\}$. The construction of the classes $\cL_2$ and $\cL_3$ is much more intricate; the reader is referred to Appendix~\ref{app:teaching}.
\end{proof}

Knowing that $\GTD\le 1$ for countable collections, it is natural to ask whether the same holds for $\SGTD$, $\STT$, and $\CGTD$. The following theorem shows that these parameters can take any positive integer value even for countable classes.

\begin{restatable}{theorem}{restateAllStableCertifiableDimensions}\label{thm:all-stable-certifiable-dimensions}
For every integer $n\geq1$, there exists a countable collection
$\cL_n$ of infinite languages over a countable domain such that
$\SGTD(\cL_n)=\STT(\cL_n)=\CGTD(\cL_n)=n$.
\end{restatable}

\begin{proof} \emph{(Sketch.)} The claim can be verified by defining disjoint countably infinite sets $A=\{a_1,a_2,\ldots\}$ and $B=\{b_1,b_2,\ldots\}$, setting $\cX=A\sqcup B$, and 
$\cL_n=\{\cX\}\cup\{A\cup F:F\subseteq B,\ |F|\leq n-1\}\cup\{B\cup F:F\subseteq A,\ |F|\leq n-1\}$. See Appendix~\ref{app:teaching} for details.
\end{proof}

For finite collections, however, we show in Appendix~\ref{app:teaching} that all four teaching-related parameters for the generation setting are at
most one. By Theorem~\ref{thm:dimension-hierarchy}, it suffices to prove
this bound for $\CGTD$.

\begin{restatable}{theorem}{restateTeachingFinite} \label{thm:teaching-finite}
For every finite collection $\cL$, the dimensions $\GTD(\cL)$,
$\SGTD(\cL)$, $\STT(\cL)$, and $\CGTD(\cL)$ all equal $0$ if
$\bigcap_{L\in\cL}L$ is infinite, and $1$ otherwise.
\end{restatable}

\subsection{$\GTD$ versus Complexity of Identification and Uniform Generation}

We conclude this section by comparing the complexity of teaching for generation with the complexity of uniform generation and the complexity of teaching for identification. As discussed before Theorem~\ref{thm:dimension-hierarchy}, we use $\ITD$ as a notion of teaching dimension for identification; recall that $\ITD\le \NCTD^+$. We provide very simple arguments showing that, even for finite collections of regular languages, (i) uniform generation and (ii) identification can demand many examples, while certifiable generation requires at most one carefully chosen example for each target. 



\begin{restatable}{theorem}{restateDFAUnifCGTDNCTD} \label{thm:DFA-unif-CGTD-NCTD}
    (i) There exists $k \in \mathbb{N}$ such that, for every $\ell, n \in \mathbb{N}$, there exists a family $\mathcal{F}$ of infinite binary regular languages such that $|\mathcal{F}|=2^{\ell}$, each $L \in \mathcal{F}$ is accepted by a DFA with at most $k \cdot \ell \cdot n$ states, and $\mathcal{F}$ is not $2^{n^{|\mathcal{F}|}}$-uniformly generatable, while  $\CGTD(\cF) = 1$.
    
        (ii) For every $\ell \in \Nat$, there exists a finite family $\cF$ of infinite binary regular languages such that $|\cF|=2^\ell$, every $L\in\cF$ is recognizable by a DFA with at most $\ell+1$ states, and $\NCTD^{+}(\cF) \geq \ITD(\cF)\ge |\cF|-1,$ while $\CGTD(\cF) = 0$ and $\cF$ is \fmreplace{1}{0}-uniformly generatable.   
\end{restatable}

\begin{proof} \emph{(Sketch.)}
    The first statement can be derived from the proof of Theorem~4.1 by \cite{arenas2025language}. The second statement is witnessed by the class $\{\Sigma^*\}\cup\{\Sigma^*\setminus\{w\}\ :\ |w|\fmreplace{\le}{<} l\}$. One can show that any tell-tale for $\Sigma^*$ in this class has size at least $2^l-1$. See Appendix~\ref{app:teaching} for details.
 \end{proof}

\section{Conclusion}
We introduced certifiable language generation, which requires a generator to eventually recognize that its output is valid. We studied it alongside the stability requirement of \cite{kalavasis2025limits}. Our finite-witness characterizations of stable and certifiable generation resolve the open question of characterizing stable generation, and they place both notions in a strict hierarchy between uniform generation and generation in the limit. Natural families such as erasing pattern languages witness some of these separations. With queries, stability imposes no restriction beyond generation in the limit. Certifiability remains restrictive: some collections are identifiable in the limit yet cannot be certifiably generated even with queries. We initiated the study of teaching sets for generation, and showed that stronger success criteria can require arbitrarily larger teaching sets, in contrast to the bound of one for generation in the limit. We also showed that, already for finite families of regular languages, teaching for certifiable generation can be far cheaper than uniform generation or identification. Interesting future directions include resource-bounded versions of our notions in the query feedback framework,  and certification guarantees under noisy or adversarially corrupted data.

\paragraph{AI Disclosure.} Generative AI tools were used to assist with working out details of some of our technical proofs, with drafting small parts of the writing, as well as with checking earlier drafts for technical errors and typos. The authors assume full responsibility for the entire content and presentation of the paper.

\printbibliography

\newpage

\appendix

\section{Proofs Omitted from Section~\ref{sec:certifiability}}\label{app:stab-cert}

\restateCertChar*

\begin{proof}
For the only-if direction, fix $K\in\cL$ and any enumeration $\bar{\sigma}$ of $K$. Let $m$ be a
round at which certification occurs, and set $T_K=\supp(\bar{\sigma}_m)$. Soundness
gives
\[
\cG(\bar{\sigma}_m)\subseteq\bigcap\nolimits_{L\in\cL:T_K\subseteq L}L.
\]
The left-hand side is infinite, so \eqref{eq:certificate} holds.

For the if direction, suppose that every $K\in\cL$ contains a finite set $T_K$ satisfying
\eqref{eq:certificate}. For a finite sequence $\sigma \in \cX^*$, consider
\[
I(\sigma):=\bigcap\nolimits_{L\in\cL:\supp(\sigma)\subseteq L}L.
\]
If at least one language is consistent with $\sigma$ and $I(\sigma)$ is infinite,
define $c(\sigma)=1$ and $\cG(\sigma)=I(\sigma)$. Otherwise, define $c(\sigma)=0$ and
$\cG(\sigma)=\cX$. This procedure is sound. 

Fix $K\in\cL$ and an enumeration $\bar{\sigma}$ of $K$. Since $T_K$ is finite, there
is a round $m$ such that $T_K\subseteq\supp(\bar{\sigma}_n)$ for every $n\geq m$.
For every such $n$, any language containing $\supp(\fmreplace{\sigma}{\bar \sigma}_n)$ also contains
$T_K$. Therefore,
\[
I(\bar{\sigma}_n)\supseteq\bigcap\nolimits_{L\in\cL:T_K\subseteq L}L.
\]
The right-hand side is infinite by assumption.
Hence $c(\bar{\sigma}_n)=1$
for every $n\geq m$, which proves eventual certification. 
\end{proof}

\restateGenerationHierarchy*

\begin{proof}
Suppose first that $\cL$ is $i$-uniformly generatable by $\cG$. Set
$c(\sigma)=1$ whenever $|\sigma|\geq i$, and set $c(\sigma)=0$ otherwise. If
$\sigma \in \cX^*$ has length at least $i$, uniform generation guarantees
$\cG(\sigma)\subseteq L$ for every $L\in\cL$ consistent with $\sigma$. Thus this
certification rule is sound. Every target is infinite, so its enumeration
eventually reaches the threshold. Therefore $\cL$ is certifiably generatable.

Next, suppose that $\cL$ is certifiably generatable. Modify the generator
so that, after its first certification, it retains the certified output.
On an enumeration of a target $K$, certification occurs after a finite
number of rounds, and soundness implies that the retained infinite set
is contained in $K$. The modified generator is therefore stable. Finally,
a stable generator is eventually correct, so stable generation implies
generation in the limit.

A collection that is certifiably generatable but not uniformly
generatable can be defined as follows: For every $i\in\Nat$, define $A_i:=\{2^i(2k-1)-1:k\in\Nat\}.$
These are infinite sets of odd numbers. They are pairwise disjoint because,
for every $a\in A_i$, the largest power of two dividing $a+1$ is exactly
$2^i$. Now consider
\[
\cL_1:=\{L_i:i\in\Nat\},\qquad
L_i:=\{2,4,\ldots,2i\}\cup A_i.
\]
We first show that $\cL_1$ is certifiably generatable. Fix a target
$L_i$ and choose any $a\in A_i$. If $a\in L_j$, then, since $a$ is odd,
it must belong to $A_j$. Disjointness of the sets $A_j$ implies $j=i$.
Thus the only language in $\cL_1$ containing $a$ is $L_i$, and $\bigcap_{L\in\cL:\{a\}\subseteq L}L=L_i.$
This intersection is infinite. Therefore $\{a\}$ is a certifying witness
for $L_i$, and Theorem~\ref{thm:cert-char} gives certifiable generation.
 On the other hand, notice that for every $m \in \Nat$ we have 
\[ 
|\bigcap\nolimits_{j\geq m}L_j| = |\{2,4,\ldots,2m\}| = m
\] 
This implies that $\cL_1$ is not $m$-uniformly generatable. Thus, $\cL_1$ is not uniformly generatable.


A collection that is stably generatable 
but not certifiably generatable is given by $\cL_2:=\{L_i := \Nat \setminus \{i\} : i \in \Nat\}$ over $\cX = \Nat$. Define $\cG(\sigma):=L_{j(\sigma)}$, where $j(\sigma):=\min\{j:j\notin\supp(\sigma)\}.$
For a target $L_i$, all of $\{1, ..., i - 1\}$ eventually
appears, whereas $i$ never appears. From that point onward,
$j(\sigma_n)=i$, so the generator always outputs $L_i$. Thus $\cL_2$
is stably generatable. (Since the output eventually equals the target
language itself, the same map also identifies $\cL_2$ in the limit.)
On the other hand, for every finite $T\subseteq\Nat$, we have
$
\bigcap_{L\in\cL_2:T\subseteq L}L
=
\bigcap_{i \notin T}
\bigl(\Nat\setminus\{i\}\bigr)
=T$.
This intersection is finite, so Theorem~\ref{thm:cert-char}
implies that $\cL_2$ is not certifiably generatable.

Finally, Theorem 8.7 of \citet{kalavasis2025characterizations} showed that the collection $\cL_3 = \fmreplace{\cX}{\{\cX\}} \cup \{\cX \setminus \{w\}: w \in \cX\}$ is not stably generatable, but it is generatable in the limit due to countability \citep{kleinberg2024language}.
\end{proof}

\restateTfPatterns*

\begin{proof}
   Reidenbach \cite{REIDENBACH200691} proved that, when  $|\Sigma|=2$, the class $\eTFEP_\Sigma$ is not (even non-computably) identifiable in the limit. To see that $\eTFEP_\Sigma$ is not $1$-uniformly generatable, note that the first
example in an enumeration of a target language may be the empty word $\varepsilon$, which belongs to
every terminal-free erasing pattern language, while
$\bigcap_{r\geq1}L_E(x^r)=\{\varepsilon\}$.
Thus no infinite output can be valid for every language consistent
with this first example.

On the other hand, consider the first two examples in any enumeration of any target language $L\in\eTFEP_\Sigma$. One of these two examples must be a nonempty word $w=a_1\cdots a_m\in L$ for some $m\ge 1$ and $a_i\in\Sigma$. We claim that $W:=\{a_1^k\cdots a_m^k:k\geq1\}$ is an infinite subset of $L$. Indeed, take a substitution generating $w$ and replace each letter
$c$ in its variable images by $c^k$. Since the pattern is
terminal-free, this produces $a_1^k\cdots a_m^k$. These words have
distinct lengths, so $W$ is infinite.

The generator can therefore output $W$ and certify its validity. Hence
the collection is $2$-uniformly generatable. Moreover, any nonempty
word in a target language forms a singleton certifiable teaching
set, implying certifiable generatability.
\end{proof}

\restatePatterns*

\begin{proof} For $|\Sigma|\in\{2,3,4\}$, Reidenbach \cite{REIDENBACH2008166} proved that $\eFEP_\Sigma$ is not (even non-computably) identifiable in the limit. (Reidenbach's proof goes through even when excluding patterns that have no variables.)
We first prove stable generatability using two-word witnesses, and then prove the impossibility of uniform generation using finite intersections of pattern languages.

\textbf{Stable generation.} Fix a shortlex ordering $\prec$ on $\Sigma^*$. For every
$L\in\eFEP_{\Sigma}$, let $s_L$ and $w_L$ denote,
respectively, the first and second elements of $L$ according
to $\prec$, and define $T_L:=\{s_L,w_L\}.$

We verify the two conditions in Theorem~\ref{thm:stable-char}, starting with the second condition. For $L, L' \in \eFEP_{\Sigma}$, suppose that $T_L\cup T_{L'}\subseteq L\cap L'.$
Since $s_L\in L'$ and $s_{L'}\in L$, by minimality of the first elements we have $s_{L'}\preceq s_L$, and $s_L\preceq s_{L'},$
and hence $s_L=s_{L'}=:s.$ Similarly, $w_L\in L'$ and $w_{L'}\in L$. Since these are the
smallest elements strictly larger than $s$ in the respective
languages, $w_L=w_{L'}.$
Thus $T_L=T_{L'}.$ 

Consider any $s, w \in \Sigma^*$, and any $L=L_E(\alpha)\in \eFEP_{\Sigma}$ such that $s_L = s$ and $w_L = w$, i.e., $T_{L} = \{s, w\}$.
The word obtained by erasing every variable from $\alpha$
is the unique shortest word of $L$; hence it is $s$.

For every variable $x\in\operatorname{var}(\alpha)$, let
$|\alpha|_x$ denote the number of occurrences of $x$ and let
$r:=\min_{x\in\operatorname{var}(\alpha)}|\alpha|_x.$
Every word in $L\setminus\{s\}$ has length at least
$|s|+r$, and this bound is attained by mapping a variable
$x$ with $|\alpha|_x=r$ to a single terminal symbol and
erasing every other variable. Therefore, $|w|=|s|+r.$

Moreover, every substitution producing a word of this length
maps exactly one variable $x$ with $|\alpha|_x=r$ to a
single symbol $c\in\Sigma$ and maps all other variables to
the empty word. Thus $w$ is obtained from $s$ by
inserting exactly $r$ copies of one symbol $c$.

Write
\[
s
=
c^{q_0}a_1c^{q_1}a_2\cdots a_m c^{q_m},
\qquad a_i\neq c.
\]
Then, for uniquely determined $d_0,\ldots,d_m\geq0$ with
$\sum_i d_i=r$,
\[
w
=
c^{q_0+d_0}a_1c^{q_1+d_1}a_2\cdots
a_m c^{q_m+d_m}.
\]
For every $k\geq0$, define
\[
b_k(s,w)
:=
c^{q_0+kd_0}a_1c^{q_1+kd_1}a_2\cdots
a_m c^{q_m+kd_m}.
\]
Mapping $x$ to $c^k$ and every other variable to the empty
word shows that for every $k\geq 0$, $b_k(s_L,w_L)\in L$. Hence
\[
    B_{s_L,w_L}
    :=
    \{b_k(s_L,w_L):k\geq0\}
\]
is an infinite subset of $L$.

Since $L$ is chosen arbitrarily, this implies
$\left|\bigcap_{L: T_L = \{s, w\}} L\right| = \infty.$
This proves the first condition.

\textbf{Impossibility of uniform generation.}
Choose distinct symbols $a,b\in\Sigma$, and consider $P:=L_E(xaxb),$ and $Q:=L_E(yaby).$
We first prove that
\[
P\cap Q=R:=\{b^j a b^{j+1}:j\geq0\}.
\]
A word in the intersection has representations $u a u b=v a b v$.
Comparing lengths gives $|u|=|v|$, and comparing the initial blocks of
this length gives $u=v$. Cancelling the common prefix $u a$ then gives
$ub=bu$, which forces $u$ to consist only of copies of $b$. Conversely,
substituting $u=v=b^j$ implies that $b^j a b^{j+1}$ is in both languages and establishes the displayed equality.

Fix any $m\geq1$, and define
\[
R_m:=\{b^j a b^{j+1}:0\leq j<m\},\qquad
\alpha_m:=z_1^1\cdots z_{m-1}^{m-1}\,a\,
          z_0^1z_1^2\cdots z_{m-1}^{m},
\]
where $z_0,\ldots,z_{m-1}$ are distinct variables. We claim that $R\cap L_E(\alpha_m)=R_m.$
Observe that proving this claim shows that $\eFEP_\Sigma$ is not $m-$uniformly generatable. 
Since $m$ was chosen arbitrarily this proves that $\eFEP_\Sigma$ is not uniformly generatable.

For each $0\leq j<m$, map $z_j$ to $b$ and erase every other variable.
This produces $b^j a b^{j+1}$, so $R_m\subseteq R\cap L_E(\alpha_m)$.

We now prove $R\cap L_E(\alpha_m) \subseteq R_m$. Consider a substitution of $\alpha_m$ producing a word in
$R$. The word has exactly one occurrence of $a$, already defined by the
terminal in $\alpha_m$, and all its other symbols are $b$. Every variable
must therefore be substituted by $b^{\ell_j}$ for some $\ell_j\geq0$. The numbers
of $b$'s before and after the terminal $a$ are, respectively,
$\sum_{j=0}^{m-1}j\ell_j$ and 
$\sum_{j=0}^{m-1}(j+1)\ell_j$.
For a word in $R$, their difference must be one. Thus
$\sum_{j=0}^{m-1}\ell_j=1$, i.e., exactly one variable $z_j$ for $j<m$ is mapped to $b$, and
all other variables are erased. This will imply that the word is in $R_m$, proving the claim.
\end{proof}

\section{Proofs Omitted from Section~\ref{sec:queries}}\label{app:query}


\restateStableQuery*

\begin{proof}
The equivalence of (ii) and (iii) is Theorem~3.4 of \citet{hanneke2026language}. The following proof makes the stable-support
conclusion explicit. Furthermore, (i) $ \Rightarrow$ (ii) follows directly from the definition:
a stable output that is contained in the target is, in particular,
eventually correct. It remains to prove (iii) $\Rightarrow$ (i).

Suppose that $\cL$ has a countable inner-cover $\{C_1,C_2,\ldots\}$.
A finite nonempty cover can be listed in this form by repeating its
elements. Fix an enumeration $u_1,u_2,\ldots$ of $\cX$, and query $u_n$
at round $n$. Let
\[
N_n:=\{u_t:t\leq n,\ a_t=0\}
\]
be the set of negatively answered query points. Define
\[
i_n:=\min\{i\geq1:C_i\cap N_n=\varnothing\},\qquad
\cG(\fmreplace{\sigma}{\bar \sigma}_n,\fmreplace{\alpha}{\bar \alpha}_n):=C_{i_n}.
\]
This minimum exists on every realizable interaction. Indeed, the target
contains some cover element $C_i$, and no point of $C_i$ can receive a
negative answer. Thus at least one index is always eligible. On an
unrealizable \azreplace{input}{interaction} for which there is no eligible index, define the
output arbitrarily as $\cX$.

Fix a target $K\in\cL$, and let
\[
i^*:=\min\{i\geq1:C_i\subseteq K\}.
\]
The index $i^*$ is never excluded, so $i_n\leq i^*$ at every round.
For each $i<i^*$, the minimality of $i^*$ implies $C_i\not\subseteq K$.
Choose a point $z_i\in C_i\setminus K$. The fixed query sequence eventually
queries $z_i$ and receives a negative answer. From that round onward,
$C_i\cap N_n\ne\varnothing$, so index $i$ is permanently excluded.

There are only finitely many indices below $i^*$. Hence there is a finite
round $n_0$ by which all of them have been excluded. Since $i^*$ remains
eligible, for every $n\geq n_0$ we have
\[
i_n=i^*,\qquad \cG(\fmreplace{\sigma}{\bar \sigma}_n,\fmreplace{\alpha}{\bar \alpha}_n)=C_{i^*}\subseteq K.
\]
The output is therefore constant and valid from round $n_0$ onward. This
proves stable generation. 
\end{proof}

\restateCertifiableQueryFeedback*

\begin{proof}
For necessity, fix $K\in\mathcal L$ and an enumeration of $K$.
Take a finite transcript at which certification occurs.
Let $P_K$ consist of the observed positive examples and
positive query points, and let $N_K$ consist of the
negative query points.
Every language in $\mathcal V_{\mathcal L}(P_K,N_K)$
is consistent with this transcript.
Soundness therefore places the certified infinite output
in their common intersection.

For sufficiency, fix an enumeration $u_1,u_2,\ldots$ of
$\cX$ and query $u_t$ at round $t$.
Let $P_t,N_t$ record all positive and negative information
available at that time.
Define
\[
I_t
:=
\bigcap_{L\in\mathcal V_{\mathcal L}(P_t,N_t)}L
\]
Let the output of generator be arbitrary whenever $I_t$ is finite, and $I_t$ whenever it is infinite. Furthermore, certify whenever, $I_t$ is infinite.
On every realizable transcript, the target belongs to the version space, so certification is sound.

Fix a target $K$ and its finite witnesses $P_K,N_K$. Since the queries exhaust $\mathcal{X}$ in the fixed order, every element of $P_K \cup N_K$ is eventually queried at a time $t$.
At that time,
\[
\mathcal V_{\mathcal L}(P_t,N_t)
\subseteq
\mathcal V_{\mathcal L}(P_K,N_K),
\]
and hence
\[
I_t
\supseteq
\bigcap_{L\in\mathcal V_{\mathcal L}(P_K,N_K)}L.
\]
The latter intersection is infinite, so certification occurs
after finitely many rounds.
\end{proof}


\restateCompleteGenHierarchy*

\begin{proof}
Theorem~\ref{thm:generation-hierarchy} establishes the three
strict inclusions between uniform, certifiable, stable, and
generation in the limit. Propositions~\ref{prop:query-cert-implies-stable}
and~\ref{prop:stable-not-query-cert} establish the strict inclusion
of certifiable generation with query feedback in stable generation
without feedback. The strictness of inclusion of generation in limit in generation in limit with query-feedback is proven in Theorem 3.10 of \citet{hanneke2026language}. The equivalence of stable and ordinary generation
with query feedback follows from Theorem~\ref{thm:stable-query}. 

We establish the remaining relations.

\emph{Certifiable generation versus certifiable generation with
query feedback.}
By Theorem~\ref{thm:generation-hierarchy} the collection $\cL:=\{L_i := \Nat \setminus \{i\}: i \in \Nat\}$ is not
certifiably generatable without feedback. However, for each
target $L_i=\Nat \setminus\{ i \}$, let
$N_i=\{i\}$ and $P_i = \varnothing$. Then $\cV_{\cL}(P_i, N_i) = L_i.$ Theorem~\ref{thm:certifiable_query_feedback} implies that the collection is certifiably generatable with query feedback.
The reverse inclusion follows by ignoring queries.

\emph{Identification in the limit.}
Identification implies stable generation by outputting the
language hypothesized by the identifier. The inclusion is strict
by Theorem~\ref{thm:patterns}, which gives a stably generatable
collection that is not identifiable in the limit. Furthermore, Theorem~\ref{thm:tf-patterns} gives a uniformly
generatable collection that is not identifiable in the limit.
Conversely, the collection in
Proposition~\ref{prop:stable-not-query-cert} is also identifiable in the
limit, but is not
certifiably generatable even with query feedback. 

Therefore, identification is incomparable with uniform generation,
certifiable generation, and certifiable generation with query
feedback. This proves all the claimed relations.
\end{proof}

\begin{restatable}{proposition}{restateQueryCertImpliesStable}
\label{prop:query-cert-implies-stable}
    If a collection $\mathcal L$ is certifiably generatable with query
feedback, then it is stably generatable even without feedback.
\end{restatable}

\begin{proof}
For finite disjoint sets $P,N\subseteq\mathcal X$, let $I(P,N):=\bigcap_{L\in \fmreplace{V}{\cV}_{\mathcal L}(P,N)}L.$
Define the family of finite pairs that can witness the characterization in Theorem~\ref{thm:certifiable_query_feedback} as
\[
\cW:=\bigl\{(P,N):
P,N\subseteq\cX\text{ are finite and disjoint},\;
\fmreplace{V}{\cV}_{\cL}(P,N)\neq\varnothing,\;
|I(P,N)|=\infty\bigr\}.
\]
Observe that since $\mathcal X$ is countable, $\mathcal W$ is also countable. Let $(P_i,N_i)_{i\geq 1}$ be an enumeration of $\cW$  and let $C_i:=I(P_i,N_i)$.
By definition, every $C_i$ is infinite. Moreover, by the characterization
of certifiable generation with query feedback, for every
$K\in\mathcal L$ there is an index $i$ such that
$P_i\subseteq K$ and $N_i\cap K=\varnothing$.

We now define the generator without feedback. For a finite sequence $\fmreplace{\sigma}{\bar \sigma}$, let
$S=\supp(\fmreplace{\sigma}{\bar \sigma})$ and define $ i(\fmreplace{\sigma}{\bar \sigma})=\min\{i\geq 1:
P_i\subseteq S,\;N_i\cap S=\varnothing\}$,
whenever the defined set is not empty. The generator is defined as
\[
\fmreplace{G}{\cG}(\fmreplace{\sigma}{\bar \sigma}):=
\begin{cases}
C_{i(\fmreplace{\sigma}{\bar \sigma})}, & \text{if $i(\fmreplace{\sigma}{\bar \sigma})$ is defined},\\
\mathcal X, & \text{otherwise}.
\end{cases}
\]

Fix a target language $K\in\mathcal L$ and an enumeration
$\fmreplace{\sigma}{\bar \sigma}$ of $K$. Let
\[
i^*:=\min\{i\geq 1:
P_i\subseteq K,\;N_i\cap K=\varnothing\}.
\]
Since $P_{i^*}$ is finite, every enumeration of $K$
eventually contains $P_{i^*}$. Furthermore, no element
of $N_{i^*}$ appears in the enumeration. Consequently, there will be some round such that $P_{i^*}$ is a subset of the instances appeared so far and  $N_{i^*}$ has empty intersection with the input, meaning that the generator will consider this as an index in the set it is minimizing.

Now consider any $i<i^*$. By the minimality of $i^*$,
either $P_i\not\subseteq K$ or $N_i\cap K\neq\varnothing$.
In the first case, $P_i$ is never contained in the observed
examples, so $i$ is never considered by the generator. In the second case,
choose $x_i\in N_i\cap K$. Since $\fmreplace{\sigma}{\bar \sigma}$ enumerates $K$,
there is a finite round at which $x_i$ appears. From that
round onward,
$N_i\cap\operatorname{supp}(\fmreplace{\sigma}{\bar \sigma}_n)\neq\varnothing$,
so $i$ will never be considered.

There are only finitely many indices smaller than $i^*$.
Thus there is a round $n^*$ such that
$i(\fmreplace{\sigma}{\bar \sigma}_n)=i^*$ for every $n\geq n^*$.
Since $K\in \fmreplace{V}{\cV}_{\mathcal L}(P_{i^*},N_{i^*})$, we have
\[
\fmreplace{G}{\cG}(\fmreplace{\sigma}{\bar \sigma}_n)=C_{i^*}
=\bigcap_{L\in \fmreplace{V}{\mathcal V}_{\mathcal L}(P_{i^*},N_{i^*})}L
\subseteq K
\qquad\text{for every }n\geq n^*.
\]
The output is therefore constant and valid from round
$n^*$ onward, proving stable generation without feedback.
\end{proof}

We now show that the above implication is strict, namely, there is a countable collection that is stably generatable but not certifiably generatable even with query feedback.

\begin{restatable}{proposition}{restateStableNotQueryCert} 
\label{prop:stable-not-query-cert}
There exists a countable collection that is both stably generatable without feedback and identifiable in the limit but not certifiably generatable with query
feedback.
\end{restatable}

\begin{proof}
Let $A=\{a_1,a_2,\ldots\}$ and $M=\{m_1,m_2,\ldots\}$
be disjoint countably infinite sets, and let
$\mathcal X=A\sqcup M$. Define
\[
K_\star:=A,\qquad
K_j:=(A\setminus\{a_j\})\cup\{m_j\}\quad (j\geq 1),
\qquad
\mathcal L:=\{K_\star\}\cup\{K_j:j\geq 1\}.
\]
Note that every language in $\cL$ is infinite.

We first show that $\mathcal L$ is stably generatable.
For a finite sequence $\sigma$, define
$J(\sigma):=\{j\geq 1:m_j\in\operatorname{supp}(\sigma)\}$,
and let
\[
\fmreplace{G}{\cG}(\sigma):=
\begin{cases}
K_\star, & J(\sigma)=\varnothing,\\
K_{\min J(\sigma)}, & J(\sigma)\neq\varnothing.
\end{cases}
\]
On every enumeration of $K_\star$, no element from $M$
will appear, and hence the generator outputs $K_\star$ in every round. 
For any target $K_j$, the element $m_j$ eventually appears in the enumeration.
Since $K_j \cap M = \{m_j\}$ from that round
onward the generator outputs $K_j$ and never changes
its output. Thus $\cL$ is stably generatable
without feedback. Since the output of the generator eventually equals the target language itself, the same map also identifies $\cL$ in the limit.

We next show that $\mathcal L$ is not certifiably
generatable with query feedback. Consider the target
$K_\star=A$ and any finite sets $P\subseteq A$ and
$N\subseteq M$. By construction,
\[
\fmreplace{V}{\mathcal V}_{\mathcal L}(P,N)
=\{K_\star\}\cup
\{K_j:a_j\notin P,\;m_j\notin N\}.
\]
Since $K_j$ excludes $a_j$ and $K_\star$ excludes
every element from $M$, we obtain
\[
\bigcap_{L\in \fmreplace{V}{\mathcal{V}}_{\mathcal L}(P,N)}L
=P\cup\{a_j:m_j\in N\}.
\]
The right-hand side is finite. Therefore no finite sets $P$ and $N$ can be a witness to the  characterization of certifiable generation for \fmreplace{$K^*$}{$K_\star$}. Hence, $\cL$ is not certifiably generatable with query feedback.
\end{proof}


\section{Proofs Omitted from Section~\ref{sec:teaching}}\label{app:teaching}

\restateDimensionHierarchy*

\begin{proof}
We first prove the chain of inequalities. 
Every stable generation teaching set is a generation teaching set, giving
the first inequality.

Next, consider any admissible assignment
$\mathbf T=(T_L)_{L\in\cL}\in\mathfrak A(\cL)$. In the if direction
of Theorem~\ref{thm:stable-char}, we constructed a generator $\cG$
that satisfies $\SGTD(\cL,\cG)\leq\sup_{L\in\cL}|T_L|.$
which proves the second inequality.

To prove $\STT(\cL)\leq\CGTD(\cL)$, it suffices to consider the case
$\CGTD(\cL)=m<\infty$. Fix an enumeration of all finite sets of size at most $m$ that satisfy the property in \eqref{eq:certificate}. Such sets form a countable family because $\cX$ is
countable. For every $L\in\cL$, assign the first set in this enumeration
that is contained in $L$, and denote it by $T_L$. The definition of
$\CGTD$ ensures that this choice exists.

We show that the assignment $(T_L)_{L \in \cL}$ satisfies \eqref{eq:stable-fiber} and \eqref{eq:stable-noclash}. 
Consider any $T_L$ and $T_{L'}$ such that $T_{L} \cup T_{L'} \subseteq L, L'$. Then each assigned
set was an available choice for both languages. Since we chose the first available set according \fmedit{to} a fix enumeration for both languages, we must have
$T_L=T_{L'}$. Therefore, the assignment satisfies \eqref{eq:stable-noclash}.  Furthermore, we have 
\[
\bigcap\nolimits_{L\in\cL:T_L=T}L
\supseteq
\bigcap\nolimits_{L\in\cL:T\subseteq L}L,
\]
Consequently, the assignment satisfies \eqref{eq:stable-fiber}. Therefore $\STT(\cL) \leq \sup_{L \in \cL} |T_L| \leq m,$ proving the last inequality.

For the first separation, use the collection
$\cL_1=\{\cX\}\cup\{\cX\setminus\{w\}:w\in\cX\}$ which is also used by \citet{kalavasis2025characterizations}. This collection is not stably generatable, so
$\SGTD(\cL_1)=\infty$. On the other hand, since $\cL_1$ is countable, Theorem~\ref{thm:gtd} gives
$\GTD(\cL_1)\leq1$.



The second and third separations are witnessed by the 
constructions in
Corollary~\ref{cor:infinite-gap} and Proposition~\ref{prop:stt-cgtd-separation}.
\end{proof}

Throughout this appendix, the term ``admissible assignment'' refers to an element of
$\mathfrak A(\cL)$, as in Definition~\ref{def:ASGTD}. 

 The following construction witnesses the third separation in
Theorem~\ref{thm:dimension-hierarchy}, between $\STT$ and $\CGTD$.


\begin{proposition}\label{prop:stt-cgtd-separation}
There exists a countable collection $\cL$ such that $\STT(\cL)=1$ and $\CGTD(\cL)=\infty$.
\end{proposition}

\begin{proof}
For $1\leq i<j$, define
\[
L_{i,j}:=
\{2k-1:1\leq k<i\}
\cup\{2,4,\ldots,2j\}
\cup\{2k-1:k\geq j\},
\]
and let $\cL:=\{L_{i,j}:1\leq i<j\}$. This collection is countably
infinite, and every language contains an infinite odd tail.

We first show that $\STT(\cL)\leq1$. Assign
$T_{L_{i,j}}:=\{2j\}$. Every language assigned $\{2j\}$ contains the
odd tail starting at $2j-1$, so
\[
\{2k-1:k\geq j\}
\subseteq
\bigcap_{L\in\cL:T_L=\{2j\}}L.
\]
Thus the languages assigned the same set have infinite intersection.
Furthermore, if $j<j'$, then $2j'\notin L_{i,j}$, since the largest
even number in $L_{i,j}$ is $2j$. Hence two languages assigned different
sets cannot both contain one another's assigned set. The assignment is
therefore admissible, proving $\STT(\cL)\leq1$. Since the intersection of languages in the class is not infinite, we have  $\STT(\cL) = 1$.

We next show that no target has a finite certifiable teaching set.
Fix any finite $T\subseteq\Nat$, and choose an integer
$m>\max(T\cup\{0\})$. Then $T\subseteq L_{m,j}$ for every $j>m$.
Consequently,
\[
\begin{aligned}
\bigcap_{L\in\cL:T\subseteq L}L
&\subseteq \bigcap_{j>m}L_{m,j}\\
&=\{2k-1:1\leq k<m\}\cup\{2,4,\ldots,2m+2\}.
\end{aligned}
\]
Indeed, the common even elements are those at most $2m+2$, and every
odd number beyond the initial odd segment is excluded once $j$ is
large enough. The displayed intersection is finite. Since $T$ was chosen
arbitrarily, this implies that
$\CGTD(\cL)=\infty$.
\end{proof}

The following construction witnesses the second separation in
Theorem~\ref{thm:dimension-hierarchy}, between $\SGTD$ and $\STT$.

\begin{restatable}{theorem}{restateArbitraryGap}\label{thm:arbitrary-gap}
For every $n\geq1$, there is a countably infinite collection $\cL_n$ of
cofinite languages over a countable domain such that
\[
\SGTD(\cL_n)=1\qquad\text{and}\qquad\STT(\cL_n)=n.
\]
Moreover, the generator witnessing the first equality depends only on the set of
observed examples and has exactly two possible outputs. 
\end{restatable}

\begin{proof}
Fix $n\geq1$, and set
\[
r:=n-1,\qquad q:=r^2,\qquad N:=(2q+1)(q+1),\qquad
\cF:=\binom{[N]}{r}.
\]
We have $N\geq n$. When $r=0$, the family $\cF$ consists only of the
empty set. Define two finite marker sets
\[
A:=\{a_1,\ldots,a_N\},\qquad B:=\{b_F:F\in\cF\},
\]
where $a_1,\ldots,a_N$ and $b_F, F\in \cF$ are distinct elements, and countably infinite sets $U,V$ such that $A, B, U,$ and $V$ are pairwise disjoint. 
Let $M:=A\cup B$
and $\cX:=M\sqcup U\sqcup V$.

For every $i\in[N]$ and $F\in\cF$, define
\[
P_i:=\{b_F:i\notin F\}\cup A \cup U\cup V,
\]
\[
Q_F:=\{a_i:i\in F\}\cup B\cup U\cup V.
\]
Now define families $\cP_i$ and $\cQ_{\cF}$ that are obtained by allowing at most one deletion from $P_i$ and $Q_{F}$, respectively, as follows:
\[
\cP_i:=\{P_i\}\cup\{P_i\setminus\{v\}:v\in V\},
\qquad
\cQ_F:=\{Q_F\}\cup\{Q_F\setminus\{u\}:u\in U\}.
\]
Set
\[
\cL_n:=\bigcup_{i\in[N]}\cP_i\ \cup\ \bigcup_{F\in\cF}\cQ_F.
\]
This collection is countably infinite. Every language omits only finitely
many markers and at most one point of $U\cup V$, so every language is
cofinite. Furthermore, every language in a family $\cP_i$ contains all of
$U$, and every language in a family $\cQ_F$ contains all of $V$.

\textbf{Proof of $\SGTD(\cL_n)=1$.}
For a finite training sequence $\sigma\in\cX^*$, define
\[
J(\sigma):=\{i:a_i\in\supp(\sigma)\}\setminus
       \bigcup_{\substack{F\in\cF\\b_F\in\supp(\sigma)}}F,
\]
and let
\[
\cG(\sigma):=
\begin{cases}
U,&J(\sigma)\ne\varnothing,\\
V,&J(\sigma)=\varnothing.
\end{cases}
\]
This generator depends only on the observed set and always outputs one
of the two infinite sets $U,V$.

For a target $L\in\cP_i$, use the singleton teaching set $\{a_i\}$.
Consider any $\sigma\in L^*$ containing $a_i$. Observe that
$$
i \in \{i:a_i\in\supp(\sigma)\}
$$
Furthermore, any $b_F$ in any language in $\cP_i$ satisfies $i \notin \fmreplace{\cF}{F}$. Therefore $i \in J(\sigma)$. Consequently, $\cG(\sigma)=U\subseteq L$. The same reasoning applies for every 
continuation of $\sigma$ in $L$, so the output remains $U$.

For a target $L\in\cQ_F$, use the singleton teaching set $\{b_F\}$.
Every marker $a_i$ that can appear in a sequence consistent with $L$ has
$i\in F$. Once $b_F$ has appeared, the union subtracted in the definition
of $J(\sigma)$ contains all of $F$. It therefore contains every possible $i$ for any
observed $a_i$, so $J(\sigma)=\varnothing$ and $\cG(\sigma)=V\subseteq L$.
Again, the same argument holds on every consistent continuation, and the
output remains $V$. Consequently, every target has a singleton stable
generation teaching set, proving $\SGTD(\cL_n)\leq1$.

Every $u\in U$ is excluded by a language $Q_F\setminus\{u\}$, and every
$v\in V$ is excluded by a language $P_i\setminus\{v\}$. It follows that $\bigcap_{L\in\cL_n}L\subseteq M.$
This intersection is finite, so 
$\SGTD(\cL_n)=1$.

\textbf{Proof of $\STT(\cL_n)\leq n$.}
Assign
\[
T_L:=
\begin{cases}
\{a_1,\ldots,a_n\},&L\in\bigcup_i\cP_i,\\
\varnothing,&L\in\bigcup_F\cQ_F.
\end{cases}
\]
The assigned sets are contained in their targets and have size at most
$n$. All the languages that are assigned the set $\{a_1,\ldots,a_n\}$ contain $U$,
and all those assigned to  $\varnothing$ contain $V$. Hence both required
intersections are infinite and satisfy the condition  \eqref{eq:stable-fiber}.

We prove the condition  \eqref{eq:stable-noclash}. For a pair of languages to be assigned to different sets it must be that one belongs to $\bigcup_F\cQ_F$ and the other belongs to $\bigcup_i\cP_i$. However, for every $F \in \cF$, every language in $\cQ_F$ contains exactly $r=n-1$ markers from $A$. Hence, it cannot contain $\{a_1,\ldots,a_n\}$ so mutual consistency is impossible. Therefore the assignment is admissible, and $\STT(\cL_n)\leq n$.

%

\textbf{Proof of $\STT(\cL_n)\geq n$.}
Suppose, for the sake of contradiction, that there is an admissible
assignment with $|T_L|\leq n-1= r$ for every $L\in\cL_n$. For a finite set
$T\subseteq\cX$, define
\[
E(T):=\bigcup_{\substack{F\in\cF\\b_F\in T}}F.
\]
Each set $F\in\cF$ has size $r$, and an assigned teaching set contains at
most $r$ markers from $B$. Therefore
\[
|E(T_L)|\leq r|T_L|\leq r^2=q.
\]
The only instances missing from $P_j$ are the markers $b_F$ for which $j\in F$.
Consequently, for every finite set $T$,
\begin{equation}\label{eq:gap-consistency}
T\subseteq P_j\quad\Longleftrightarrow\quad j\notin E(T).
\end{equation}

Write $T_i:=T_{P_i}$. If $T_i\ne T_j$, condition~\eqref{eq:stable-noclash} implies
that at least one of $T_i\subseteq P_j$ and $T_j\subseteq P_i$ fails.
By \eqref{eq:gap-consistency}, this means $j\in E(T_i)$ or $i\in E(T_j)$.
Select $s$ indices whose teaching sets are pairwise distinct. Each of their
$\binom s2$ pairs requires at least one such membership. On the other
hand, each selected $E(T_i)$ contains at most $q$ indices and can account
for at most $q$ pairs. Hence, $\binom s2\leq sq,$
which gives $s\leq2q+1$. There are therefore at most $2q+1$ different
teaching sets among the languages $P_1,\ldots,P_N$.

Since $N=(2q+1)(q+1)$, some teaching set $T$ occurs on at least $q+1$
indices. Let
\[
I:=\{i:T_i=T\},\qquad |I|\geq q+1.
\]
We next show that there must exist some $F\in \cF$ such that $Q_F$ receives this same
teaching set. The set of indices $\{i:a_i\in T\}$ has size at most $r$,
so choose an $r$-element set $F\in\cF$ containing all of these indices.
The language $Q_F$ contains the corresponding $A$-markers and every point
of $B\cup U\cup V$. Hence $T\subseteq Q_F.$
Furthermore, $|E(T_{Q_F})|\leq q<|I|$, so there is an index
$i\in I\setminus E(T_{Q_F})$. By \eqref{eq:gap-consistency}, 
$T_{Q_F}\subseteq P_i$. We also know from the definition of $I$ that $ T_{P_i}=T$ and hence $T_{P_i}\subseteq Q_F$.
The two teaching sets are mutually consistent. Condition~\eqref{eq:stable-noclash}
therefore forces $T_{P_i}=T_{Q_F}=T$.

Fix any $v\in V\setminus T$. The language $P_i\setminus\{v\}$ contains
$T$ and is a subset of $P_i$. Its own teaching set is contained in
$P_i\setminus\{v\}$, and therefore in $P_i$. Thus $T_{P_i}=T\subseteq P_i\setminus\{v\}$ and 
$T_{P_i\setminus\{v\}}\subseteq P_i$.
Applying \eqref{eq:stable-noclash} again gives
$T_{P_i\setminus\{v\}}=T$.

Similarly, fix any $u\in U\setminus T$. Since $Q_F\setminus\{u\}$
contains $T$ and is contained in $Q_F$, we have
$T_{Q_F}=T\subseteq Q_F\setminus\{u\}$ and $T_{Q_F\setminus\{u\}}\subseteq Q_F$.
Therefore $T_{Q_F\setminus\{u\}}=T$ as well.

Thus, for every $x\in(U\cup V)\setminus T$, there is some
$L\in\cL_n$ with $T_L=T$ and $x\notin L$. Consequently,
\[
\bigcap_{L\in\cL_n:T_L=T}L\subseteq M\cup T.
\]
Since $M\cup T$ is finite, this contradicts the infinite-intersection
condition for languages assigned the same teaching set. Hence, no assignment of maximum
size $r=n-1$ is possible and $\STT(\cL_n)\geq n$. Combining the upper
and lower bounds proves $\STT(\cL_n)=n$.
\end{proof}

\begin{restatable}{corollary}{restateInfiniteGap}\label{cor:infinite-gap}
There is a countably infinite collection $\cL_\infty$ over a countable domain
such that
$\SGTD(\cL_\infty)=1$ and $\STT(\cL_\infty)=\infty$.
\end{restatable}

\begin{proof}
For each $n\geq1$, let $\cX_n$ be the domain of the collection $\cL_n$
constructed in Theorem~\ref{thm:arbitrary-gap}, and let $\cG^{(n)}$ be its
generator. Take disjoint tagged copies of these collections. Formally, let
\[
\cX_\infty:=\bigsqcup_{n\geq1}(\{n\}\times\cX_n)\qquad\text{and}\qquad
\cL_\infty:=\bigcup_{n\geq1}
\{\{n\}\times L:L\in\cL_n\}.
\]
Both the domain and the collection are countable. We first show that
$\SGTD(\cL_\infty)=1$.

On a nonempty training sequence $\sigma$ whose elements all have the same
tag $n$, remove the tags, apply $\cG^{(n)}$, and add tag $n$ to every
element of its output. 
On the empty sequence 
output any fixed infinite set. For a target $\{n\}\times L$, use the tagged
copy of the singleton teaching set for $L$. Every consistent sequence
containing this point is nonempty and lies entirely in component $n$.
The component generator is stable and correct once its singleton has
appeared, so the tagged generator has the same property. Therefore
$\SGTD(\cL_\infty)\leq1$. Languages in different components are disjoint,
so their common intersection is empty, ruling out dimension zero.

We next show that $\STT(\cL_\infty)=\infty$. Consider any admissible
assignment and restrict it to one tagged component. The no-clash condition
continues to hold, since it concerns only pairs of languages. Also,
restricting the assignment only removes languages from each collection
$\{L:T_L=T\}$, so their common intersection remains infinite. Thus
the restriction is an admissible assignment on that component.
Removing the tags does not change any teaching-set size. Consequently,
\[
\STT(\cL_\infty)\geq\STT(\cL_n)=n
\qquad\text{for every }n\geq1.
\]
No uniform finite bound can satisfy these inequalities, and hence
$\STT(\cL_\infty)=\infty$.
\end{proof}


\restateAllStableCertifiableDimensions*
\begin{proof}
Let $A=\{a_1,a_2,\ldots\},
 B=\{b_1,b_2,\ldots\}$
be disjoint countably infinite sets, and set $\cX=A\sqcup B$.
Define
\[
\mathcal A_n
=
\{A\cup F:F\subseteq B,\ |F|\leq n-1\}, \qquad
\mathcal B_n
=
\{B\cup F:F\subseteq A,\ |F|\leq n-1\},
\]
and let
$\cL_n=\{\cX\}\cup\mathcal A_n\cup\mathcal B_n.$
Every language in $\cL_n$ is infinite. Furthermore, $\cL_n$
is countable, since a countable set has only countably many
finite subsets. By Theorem~\ref{thm:dimension-hierarchy}, it suffices to show that $\CGTD(\cL_n) \leq n$ and $\SGTD(\cL_n) \geq n$.


\textbf{Proof of $\CGTD(\cL_n)\leq n$.} Define $T_A=\{a_1,\ldots,a_n\}$, $T_B=\{b_1,\ldots,b_n\}.$
Every language in $\mathcal B_n$ contains at most $n-1$
elements of $A$, so none contains $T_A$. Therefore,
\[
\bigcap_{L\in\cL_n, T_A\subseteq L}L
=
\cX \cap \left(\bigcap_{L\in\mathcal A_n}L\right)
=A.
\]
The last equality holds because every language in
$\mathcal A_n$ contains $A$, and $A$ itself belongs to
$\mathcal A_n$. Similarly, $\bigcap_{L\in\cL_n, T_B\subseteq L}L=B.$
Both intersections are infinite. Thus $T_A$ certifies every
target in $\mathcal A_n$ and the target $\cX$, while $T_B$
certifies every target in $\mathcal B_n$. Since these sets
have size $n$, we conclude that $\CGTD(\cL_n) \leq n$.





\textbf{Proof of $\SGTD(\cL_n) \geq n$.}
We first show that, for every finite $T\subseteq\cX$ with
$|T|\leq n-1$,
\begin{equation}\label{eq:two-block-small-closure}
\bigcap_{L\in\cL_n, T\subseteq L}L=T.
\end{equation}
Consider the languages $K_A=A\cup(T\cap B), K_B=B\cup(T\cap A).$
Since $|T|\leq n-1$, we have $K_A\in\mathcal A_n$ and
$K_B\in\mathcal B_n$. Both languages contain $T$, and $K_A\cap K_B=T$. Consequently,
\[
T
\subseteq
\bigcap_{L\in\cL_n, T\subseteq L}L
\subseteq K_A\cap K_B=T,
\]
proving \eqref{eq:two-block-small-closure}. This already shows that $\CGTD(\cL_n) \geq n$.

Next, for the sake of contradiction, suppose there exists a generator $\cG$ which admits
stable teaching sets $T_K\subseteq K$ of size at most $n-1$
for every $K\in\cL_n$.

Let $T=T_{\cX}$, and let $\sigma$ be a sequence listing the
elements of $T$. Since $T$ is a stable teaching set for
$\cX$, we have
\begin{equation}\label{eq:two-block-top-lock}
\cG(\sigma\cdot\rho)=\cG(\sigma)
\end{equation}
for every continuation $\rho$ over $\cX\setminus T$.

We claim that $\cG(\sigma)\subseteq H$ for every
$H\in\cL_n$ containing $T$. To prove this, fix such a
language $H$, and let $\rho_H$ list the elements of
$T_H\setminus T$. The concatenation $\sigma\cdot\rho_H$
is a sequence of distinct elements of $H$ and contains $T_H$.
Therefore, the teaching guarantee for $H$ gives $\cG(\sigma\cdot\rho_H)\subseteq H.$
On the other hand, \eqref{eq:two-block-top-lock} gives $\cG(\sigma\cdot\rho_H)=\cG(\sigma).$
Thus $\cG(\sigma)\subseteq H$, proving the claim.

It follows that $\cG(\sigma)
\subseteq
\bigcap_{H\in\cL_n, T\subseteq H}H.$
Since $|T|\leq n-1$, equation
\eqref{eq:two-block-small-closure} implies that the
intersection on the right is $T$. Hence $\cG(\sigma)\subseteq T,$
contradicting the requirement that $\cG(\sigma)$ is infinite.
Therefore $\SGTD(\cL_n)\geq n$.
\end{proof}

\restateTeachingFinite*

\begin{proof}
\textbf{Upper bound of $1$.} 
For every $A\subseteq\cL$, let $I_A:=\bigcap_{K\in A}K.$ Fix $L\in\cL$ and define
\[
B_L
=
\bigcup_{\substack{
A\subseteq\cL:\\
L\in A,\ |I_A|<\infty
}}
I_A.
\]
Since $\cL$ is finite, there are only finitely many subfamilies
$A\subseteq\cL$. Hence $B_L$ is a finite union of finite sets
and therefore is finite.
Since $L$ is infinite, choose $t_L\in L\setminus B_L$ and set
$T_L:=\{t_L\}$.

Let $A_L:=\{K\in\cL:t_L\in K\}.$
Suppose for the sake of contradiction that $I_{A_L} = \bigcap_{K\in A_L}K$
is finite. Since $L\in A_L$, the subfamily $A_L$ appears in
the definition of $B_L$, and hence $I_{A_L}\subseteq B_L.$
On the other hand, by the definition of $A_L$, $t_L\in I_{A_L}.$ Therefore $t_L\in B_L$, contradicting the choice of $t_L$.
Thus $\left|
\bigcap_{K\in\cL\ :\ T_L\subseteq K}K
\right|
=\infty.$
Hence,
$\GTD(\cL) = \SGTD(\cL) = \STT(\cL) = \CGTD(\cL) \leq 1$.
\textbf{Dimension 0 iff $\bigcap_{L\in\cL}L$ is infinite.} If $\left|\bigcap_{L\in\cL}L\right|=\infty$,
we may take $T_L=\emptyset$ for every $L\in\cL$, yielding $\GTD(\cL) = \SGTD(\cL) = \STT(\cL) = \CGTD(\cL) \leq 0$.
Conversely, if $\GTD(\cL)=0$, \fmedit{choose a generator with $\GTD(\cL, \cG) = 0$}. Then for every $L \in \cL, \cG(()) \subseteq L$. Thus, $\left|\bigcap_{L\in\cL}L\right|=\infty$. Together with Theorem~\ref{thm:dimension-hierarchy} this completes the proof.
\end{proof}

 \restateDFAUnifCGTDNCTD*

\begin{proof}
    The first statement follows from Theorem~4.1 of \cite{arenas2025language}. The family constructed there has the stated DFA complexity and is not $2^{n^{|\cF|}}$-uniformly generatable. Moreover, the common intersection of the languages in the construction is finite. Since $\cF$ is a finite collection of infinite languages, Theorem~\ref{thm:teaching-finite} implies that $\operatorname{CGTD}(\cF) = 1$.

    For the second statement, let $\Sigma=\{a,b\}$ and define
    $W = \{w\in \Sigma^*:|w|<\ell\}$. Then
    \[
    |W|
    =
    \sum_{i=0}^{\ell-1}2^i
    =
    2^\ell-1.
    \]
    
    Define $L_\star:=\Sigma^*$
    and, for every $w\in W$, $L_w:=\Sigma^* \setminus\{w\}.$
    Let
    $\cF
    =
    \{L_\star\}
    \cup
    \{L_w:w\in W\}.$
    Clearly,
    $|\cF| = 1+|W|
    = 2^\ell$. Furthermore, $\bigcap_{L \in \mathcal{F}} L=\Sigma^* \backslash W$, which implies $\CGTD(\cF) = 0$ and $\cF$ is 0-uniformly generatable.

    The language $L_\star$ is recognizable by a one-state DFA.
    Moreover, for every $w\in W$, the language
    $\Sigma^* \setminus\{w\}$ can be recognized by a DFA with
    $|w|+2\leq\ell+1$ states: the automaton keeps track of the
    longest prefix of $w$ read so far, moves to an accepting sink
    upon any mismatch, and rejects only when the input ends
    exactly after reading $w$.
    
    We next prove $\ITD(\cF) \geq 2^\ell-1$, which immediately implies $\NCTD^{+}(\cF) \geq 2^\ell-1$. 
    Let $(T_L)_{L \in \cF}$ be a tell-tale assignment.
    For every $w\in W$,
    $L_w\subset L_\star$, so that the tell-tale set $T_{L_\star}$ for $L_\star$ is not a subset of $L_w$. Thus  $w \in T_{L_\star}$ for all $w\in W$; in particular, $
    |T_{L_\star}|
    \geq
    |W|
    =
    2^\ell-1$.
    %
\end{proof}

\end{document}